\documentclass{article}
\usepackage{amssymb}

\usepackage{makeidx}
\usepackage{graphicx}
\usepackage{amsmath}

\newtheorem{theorem}{Theorem}

\newtheorem{corollary}[theorem]{Corollary}

\newtheorem{lemma}[theorem]{Lemma}

\newenvironment{proof}[1][Proof]{\textbf{#1.} }{\ \rule{0.5em}{0.5em}}

\begin{document}

\title{Decision theory in an algebraic setting}
\author{Maurizio Negri}
\maketitle

\begin{abstract}
In decision theory an act is a function from a set of conditions to the set
of real numbers. The set of conditions is a partition in some algebra of
events. The expected value of an act can be calculated when a probability
measure is given. We adopt an algebraic point of view by substituting the
algebra of events with a finite distributive lattice and the probability
measure with a lattice valuation. We introduce a partial order on acts that
generalizes the dominance relation and show that the set of acts is a
lattice with respect to this order. Finally we analyze some different kinds
of comparison between acts, without supposing a common set of conditions for
the acts to be compared.

\vspace{5mm} \noindent \textbf{Keywords:} Decision theory, lattice theory,
partitions, Allais paradox.
\end{abstract}

\section{\label{par1}Classical acts}

The concept of an act is at the basis of decision theory, in fact decision
making under risk can be reduced to the choice among different acts on the
basis of their expected value. Loosely speaking, we can say that an act is a
function from a set of conditions to a set of consequences. In this
paragraph we introduce the intuitive framework for decisions, grounded on
the concept of probability space, but in the following paragraphs we shall
adopt a more algebraic point of view, based on the concept of valued lattice.

The consequence of an act can be any kind of thing, but we confine ourselves
to elements of $R$, the set of real numbers. We only observe that in an
economical framework real numbers can represent any definite amount of
goods, money and so on, but in a psychological framework they may also
represent degrees of satisfaction, subjective feelings of pain and pleasure
and so on.

The conditions of an act are events from a probability space. A probability
space is a triple \ $(S,\mathcal{C}_{S},p)$ where $S$ is a sample space, $%
\mathcal{C}_{S}$ a field of sets over $S$ and $p:\mathcal{C}_{S}\rightarrow
\lbrack 0,1]$ a function satisfying Kolmogoroff's axioms: 1) $p(A)=1$, 2) $%
p(X\cup Y)=p(X)+p(Y)$, when $X\cap Y=\emptyset $. In this way probability is
seen as the measure of an event represented \ by a set. In the following we
shall limit ourselves to finite sample spaces, so the algebra of events $%
C_{S}$ will coincide with $\mathcal{P}(S)$, the Boolean algebra of all
subsets of $S$. The conditions of an act must satisfy a fundamental
property: they must be a partition of $S$. We say that a subset $E$ of $P(S)$
is a \textit{partition} of $S$ when the following three conditions are
satisfied:

\begin{enumerate}
\item  $\bigcup E=S$,

\item  $e_{2}\cap e_{2}=\emptyset $, for all $e_{2},e_{2}\in E$ with $%
e_{2}\neq e_{2}$,

\item  $e\neq \emptyset $, for all $e\in E$.
\end{enumerate}

\noindent One of the possible partitions is given by the set of all atoms in 
$\mathcal{P}(S)$, i.e the set of all singletons $\{s\}$ with $s\in S$.

Now we can define an \textit{act on} $E$ as a function $\alpha :E\rightarrow
R$, where $E$ is a partition of $S$. We denote with $A(E)$ the set of all
acts on $E$. The elements of $\alpha \lbrack E]$, the range of $\alpha $,
are the consequences (rewards, payoffs) of $\alpha $: intuitively, $\alpha
(e)$ is the consequence of $\alpha $ when the event $e$ happens. The choice
of a partition as the domain of an act reflects a relevant aspect of real
life acts, where we are confronted with a set of alternative and exclusive
conditions, represented by events $e_{1},...,e_{n}$, leading to consequences 
$\alpha (e_{1}),...,\alpha (e_{n})$. This amounts to say that the domain $%
\{e_{1},...,e_{n}\}$ of an act is a partition of $S$. For every state of the
world, for every experimental outcome, one and only one event $e_{i}$ of the
partition $E$ takes place leading to a single consequence $\alpha (e_{i})$.

A central problem of decision theory is to define a preference relation on
the set of acts. When $\alpha $, $\beta \in A(E)$,$\ \beta $ is obviously
preferred to $\alpha $ when it gives a better or equal reward for all
conditions and in this case we say that $\beta $ \textit{dominates} $\alpha $%
. So we define $\alpha \preceq _{E}\beta $, iff $\alpha (e)\leq \beta (e)$
for all $e\in E$. The relation $\preceq _{E}$ is a partial order on $A(E)$,
but not every pair of acts on $A(E)$ can be compared in $\preceq _{E}$, so
it is not complete. If $\alpha \in A(E)$ and $\beta \in A(D)$, where $E$ and 
$D$ are different sets of conditions, the relation of dominance is not
defined.

When we choose a dominant act, we ignore all questions about the probability
of the relevant conditions, the elements of $E$ involved in the decision
process. So we define another preference relation that takes in account the
probability of events/conditions. Given a probability measure $p:$ $\mathcal{%
P}(S)\rightarrow \lbrack 0,1]$, we define the \textit{expected value} of $%
\alpha $ \textit{with respect to} $p$ setting 
\begin{equation*}
\exp (\alpha ,p)=\sum \{\alpha (e)p(e):e\in E\}
\end{equation*}
and we define $\alpha \preceq _{\exp }\beta $ iff $\exp (\alpha ,p)\leq \exp
(\beta ,p)$. The relation $\preceq _{\exp }$ is reflexive, transitive and
complete, i.e. is a total preorder. ($\preceq _{\exp }$ is not
antisymmetric, so it is not a partial order.) The choice among different
acts, with respect to a given probability measure, is accomplished by
ranking acts by their expected value: this is the most important rule of
choice in the field of decision under risk. In general, we observe that we
may have $\exp (\alpha ,p)\leq \exp (\beta ,p)$ even if $\alpha (x)\leq
\beta (x)$ holds in a single case ($\alpha (x)\leq \beta (x)$ if $x=e$ and $%
\alpha (x)>\beta (x)$ if $x\neq e$), because the relevance of the single
condition $x\in E$ in establishing $\alpha \preceq _{\exp }\beta $ depends
on its probability value $p(x)$. A huge value $p(e)$ may rule out all other
conditions $x\in E$. Some critical remarks against the ranking of acts based
on expected utility are due to Allais and Ellsberg, in \cite{allais1953} and 
\cite{ellsberg1961}. A way out to Allais Paradox is sketched in Appendix \ref
{Appendix B}, where the notion of intrinsic expected value is introduced as
the ratio between the expected value of $\alpha $ and the total sum of all
possible rewards of $\alpha $.

The comparison of acts with respect to expected utility is generally
confined to acts on the same set of conditions. This is clear when the
representation of decisions is based on decision matrices. (See, for
instance, \cite{jeffrey1983}.) We underline, however, that every couple of
acts $\alpha :E\rightarrow R$ and $\beta :D\rightarrow R$ can be compared in 
$\preceq _{\exp }$, as far as the conditions $E$ and $D$ are partitions of
the same algebra of events. In this way we can rely on a common probability
measure $p$ and so we can compute the expected value of $\alpha $ and $\beta 
$ with respect to $p$. From the point of view of expected value, the ranking
of $\alpha $ and $\beta $ is reduced to the comparison of two real numbers $%
\exp (\alpha ,p)$ and $\exp (\beta ,p)$, leaving out every consideration
regarding the very nature of the events/conditions involved in the decision
process, even if the result may be somewhat unnatural. In this work we
introduce a partial order $\preceq _{v}$ in which acts based on different
set of conditions can be compared, as it happens with expected value. The
comparison of acts in $\preceq _{v}$is a generalization of dominance, taking
in account the partial order of the partitions involved (see Appendix \ref
{Appendix A}) and the probability of the relevant conditions.

The plan of this work is the following. In the second paragraph we analyze
the connections between acts and lotteries or gambles. In the third
paragraph we adopt an algebraic standpoint: the algebra of events $\mathcal{P%
}(S)$, the Boolean algebra of all subsets of $S$, is substituted by a finite
distributive lattice $\mathcal{A}$, the set of conditions of an act becomes
an algebraic partition of $\mathcal{A}$ and an act is a function from such a
partition to the set of real numbers. Then we introduce a partial order $%
\preceq _{v}$ on acts and in the fourth paragraph we show that the set of
all acts on $\mathcal{A}$ is a lattice with respect to $\preceq _{v}$. In
the fifth paragraph we discuss some different ways of comparing acts.

\section{Acts and lotteries}

The process of decision is often described in the literature as a choice
between lotteries or gambles: the relationships between acts and lotteries
are sketched in this paragraph. For all finite set $X$, we say that a
function $f:X\rightarrow \lbrack 0,1]$ is a \textit{distribution on} $X$
when $\sum \{f(x):x\in X\}=1$. Given a finite set $Z=\{z_{1},...,z_{n}\}$ of
real numbers, the outcomes or rewards, a \textit{lottery on}\ $Z$ is a
distribution $l:Z\rightarrow \lbrack 0,1]$. The expected value of $l$ is
defined as follows: 
\begin{equation*}
Exp(l)=\sum \{zl(z):z\in Z\}.
\end{equation*}

The concepts of lottery and act are connected but not equivalent. To every
pair constituted by an act $\alpha :E\rightarrow R$ and a probability
measure $p:$ $\mathcal{P}(S)\rightarrow \lbrack 0,1]$ we can associate a
lottery $l_{\alpha ,p}$ on $\alpha \lbrack E]$ with the same expected value,
defined as follows: for all $x\in \alpha \lbrack E]$, we set 
\begin{equation*}
l_{\alpha ,p}(x)=p(\bigcup \alpha ^{-1}(x)\}).
\end{equation*}
We show that $l_{\alpha ,p}$ is a lottery by verifying that $l_{\alpha ,p}$
is a distribution on $\alpha \lbrack E]$. If $x\in \alpha \lbrack E]$ then 
\begin{equation*}
l_{\alpha ,p}(x)=\sum \{p(e):e\in \alpha ^{-1}(x)\}
\end{equation*}
because the events in $E$ belong to a partition and are pairwise disjoint.
We have 
\begin{eqnarray*}
\sum \{l_{\alpha ,p}(x):x\in \alpha \lbrack E]\}= &&\sum \{\sum \{p(e):e\in
\alpha ^{-1}(x)\}:x\in \alpha \lbrack E]\} \\
= &&\sum \{p(e):e\in E\}=1.
\end{eqnarray*}

Now we prove that $Exp(\alpha ,p)=Exp(l_{\alpha ,p})$. In fact, 
\begin{eqnarray*}
Exp(\alpha ,p) &=&\sum \{\alpha (e)p(e):e\in E\} \\
&=&\sum \{\sum \{\alpha (e)p(e):e\in a^{-1}(x)\}:x\in \alpha \lbrack E]\} \\
&=&\sum \{\sum \{xp(e):e\in a^{-1}(x)\}:x\in \alpha \lbrack E]\} \\
&=&\sum \{x\sum \{p(e):e\in a^{-1}(x)\}:x\in \alpha \lbrack E]\} \\
&=&\sum \{xl_{\alpha ,p}(x):x\in \alpha \lbrack E]\} \\
&=&Exp(l_{\alpha ,p}).
\end{eqnarray*}

In the other direction, we cannot immediately associate an act to a lottery,
because a lottery $l$ is only a finite sequence $(l(z_{1}),...,l(z_{n}))$ of
probability values adding to $1$, without any reference to a sample space,
an algebra of events and a probability measure on this algebra. So, given a
lottery $l$ on $Z$, we must supply a finite Boolean algebra $\mathcal{P}(S)$
and a partition $E$ before defining the act $\alpha _{l}$ associated to $l$.
If $Z$ contains $n$ rewards, we choose $S=\{s_{1}...,s_{n}\}$. Then $%
E=\{\{s_{i}\}:1\leq i\leq n\}$ is a partition of $S$ and we define an act $%
\alpha _{l}:E\rightarrow R$ setting $\alpha _{l}(\{s_{i}\})=z_{i}$. Finally,
we define a probability measure $p$ on $\mathcal{P}(S)$ starting with $%
p(\{s_{i}\})=l(z_{i})$: as every event can be expressed as a disjoint union
of singletons, we can extend $p$ to the whole of $\mathcal{P}(S)$. Now we
show that $\exp (l)=\exp (\alpha _{l},p)$ where $\alpha _{l}$ and $p$ are
defined as above: 
\begin{eqnarray*}
\exp (\alpha _{l},p) &=&\sum \{\alpha
_{l}(\{s_{i}\})p(\{s_{i}\}):\{s_{i}\}\in E\} \\
&=&\sum \{z_{i}l(z_{i}):z_{i}\in Z\} \\
&=&\exp (l)
\end{eqnarray*}
\ where the second line follows because there is a bijection between $Z$ and 
$E$.

\section{Acts in finite distributive lattices}

The framework for acts introduced so far can be formulated in algebraic
terms as follows. Let $\mathcal{A}$ be a finite distributive lattice, we
introduce an algebraic counterpart of the set theoretic notion of partition
as follows: we say that $E\subseteq A$ is an \textit{algebraic} \textit{%
partition} \textit{of} $\mathcal{A}$ if

\begin{enumerate}
\item  $\bigvee E=1$,

\item  $e_{2}\wedge e_{2}=0$, for all $e_{2},e_{2}\in E$ with $e_{2}\neq
e_{2}$,

\item  $e\neq 0$, for all $e\in E$.
\end{enumerate}

We speak of a partition $E$ \textit{of} $\mathcal{A}$ ($E$ \textit{on} $A$)
when an algebraic (set-theoretical) partition is intended. We denote with $%
\Pi (\mathcal{A})$ ($\Pi (A)$) the set of all algebraic (set-theoretical)
partitions of $\mathcal{A}$($A$). The aim of 3) is to avoid redundancies. On
one side, if $E$ is a partition and $0\in E$, then $E$ is redundant because $%
E-\{0\}$ is a partition too. On the other side, if $E$ is a redundant
partition and so, for some $e\in E$, $E-\{e\}$ is a partition too, then we
can easily see that $e=0$. In fact, $e=e\wedge \bigvee (E-\{e\})=0$. To make
life easier we have collected some basic results about partitions in
Appendix \ref{Appendix A}.

Given a partition $E$ of $\mathcal{A}$, an \textit{act on} $E$ is a function 
$\alpha :E\rightarrow R$. We denote with $A(E)$ the set of all acts on $E$.
An act \textit{of} $\mathcal{A}$ is an act $\alpha :E\rightarrow R$, for
some partition $E$ of $\mathcal{A}$. We denote with $A(\mathcal{A})$ the set
of all acts of $\mathcal{A}$, i.e. $\bigcup \{A(E):E\in \Pi (\mathcal{A})\}$%
. It can be easily seen that the intuitive notion of an act, as introduced
in the first paragraph, is only a particular case of the algebraic notion.
Finally, the concept of a probability measure $p$ on $\mathcal{P}(S)$ is to
be generalized to the concept of a valuation $v$ on $\mathcal{A}$, thus
obtaining a valued lattice $(\mathcal{A},\nu )$. To make our exposition
self-contained, we introduce some basic facts about valued lattices. (See 
\cite{birkhoff1967} Chapter X.)

When $\mathcal{A}$ is a lattice we say that a function $v:A\rightarrow R$ is
a \textit{valuation} on $\mathcal{A}$ if 
\begin{equation}
v(a\vee b)=v(a)+v(b)-v(a\wedge b).  \tag{$\ast $}  \label{1}
\end{equation}
If $x\leq y$ implies $v(x)\leq v(y)$, we say that $v$ is \textit{isotone}; $%
v $ is \textit{strictly isotone} if we can substitute $\leq $ with $<$. (In 
\cite{birkhoff1967} a strictly isotone valuation is called a \textit{positive%
} one.) In the following we will confine ourselves to non-negative
valuations, i.e. valuations such that $0\leq v(a)$, for all $a\in A$. A 
\textit{valued lattice} is a pair $(\mathcal{A},v)$ where $\mathcal{A}$ is a
lattice and $v$ a valuation on $\mathcal{A}$.

If $\mathcal{A}$ is a bounded lattice, we say that $v$ is a \textit{bounded
lattices valuation} if $v$ is a valuation on $\mathcal{A}$ and $v(0)=0$ and $%
v(1)=1$. If $v$ is an isotone valuation on a bounded lattice $\mathcal{A}$,
then $v[A]\subseteq \lbrack 0,1]$. A \textit{valued bounded lattice} is a
pair $(\mathcal{A},v)$ where $\mathcal{A}$ is a bounded lattice and $v$ a
bounded lattices valuation on $\mathcal{A}$.

If $\mathcal{A}$ is a Boolean algebra, we say that $v$ is a \textit{Boolean
valuation} if $v$ is a bounded lattices valuation on $\mathcal{A}$. A 
\textit{valued Boolean algebra} is a pair $(\mathcal{A},v)$, where $\mathcal{%
A}$ is a Boolean algebra and $v$ a Boolean valuation on $\mathcal{A}$. We
can give an equivalent definition of a Boolean valuation as follows. In a
lattice with $0$, a function $f:A\rightarrow \lbrack 0,1]$ is said to be 
\textit{additive} iff $f(a\vee b)=f(a)+f(b)$, whenever $a\wedge b=0$. It can
be easily proved that, if $\mathcal{A}$ is a Boolean algebra and $%
v:A\rightarrow \lbrack 0,1]$, then $v$ is a Boolean valuation iff $v(1)=1$
and $v$ is additive. So a probability space $\ (A,\mathcal{C}_{A},p)$, where 
$A$ is a finite sample space, $\mathcal{C}_{A}$ a field of sets on $A$ and $%
p $ a probability measure satisfying Kolmogoroff's axioms with finite
additivity, is a particular case of valued Boolean algebra. It can be easily
proved that: if $(\mathcal{A},v)$ is a Boolean valued algebra then $v$ is
isotone and $v(\lnot a)=1-v(a)$. As $v$ is isotone, every Boolean valuation
takes its values in $[0,1]$ (see \cite{negri2013}, par. 2).

When $\mathcal{A}$ is a finite distributive lattice, $E$ is a partition of $%
\mathcal{A}$ and $v$ is an isotone valuation on $\mathcal{A}$, we can define
the \textit{expected value of an act} $\alpha :E\rightarrow R$ \textit{with
respect to} $v$ setting 
\begin{equation*}
\exp (\alpha ,v)=\sum \{\alpha (e)v(e):e\in E\}.
\end{equation*}
When the valuation $v$ is clear from the context, we simply write $\exp
(\alpha )$. We underline that we confine ourselves to isotone valuations, so
that $\nu (a)\in \lbrack 0,1]$ for all $a\in A$. As in the preceding
paragraph, acts can be ranked on the basis of their expected value so we
define, for all $\alpha $ and $\beta $ in $A(\mathcal{A})$, $\alpha \preceq
_{\exp }\beta $ iff $\exp (\alpha ,v)\leq \exp (\beta ,v)$. The relation $%
\preceq _{\exp }$ is reflexive, transitive and complete.

For acts having the same domain, it is natural to define an order pointwise.
Suppose $\alpha $, $\beta \in A(E)$, then we define $\alpha \preceq
_{E}\beta $ iff $\alpha (e)\leq \beta (e)$, for all $e\in E$. In this case,
we say that $\beta $ \textit{dominates} $\alpha $. It can be easily shown
that $A(E)$ is a partial order with respect to $\preceq _{E}$ and in
particular a lattice where, for all $e\in E$, $\inf (\alpha ,\beta )(e)=\min
(\alpha (e),\beta (e))$ and $\sup (\alpha ,\beta )(e)=\max (\alpha (e),\beta
(e))$. We leave a direct proof to the reader, but we observe that it is only
a particular case of the following proposition:

For all lattice $\mathcal{B}$ and all set $E$, the power $\mathcal{B}^{E}$,
where the order relation is defined pointwise, is a lattice. If $\mathcal{B}$
is distributive, bounded, complemented, so is $\mathcal{B}^{E}$.

\noindent In fact, the axioms involved are equational and then are preserved
by direct products and powers (see, for instance, \cite{ck1990} par. 6.2).
As real numbers with their natural order are a lattice, so is $A(E)$ when
acts are ordered pointwise (i.e. by dominance).

The preference relations $\preceq _{E}$ and $\preceq _{\exp }$are inspired
by different points of view. In $\alpha \preceq _{E}\beta $ acts are
compared with respect to their conditions and this is possible because $%
\alpha $ and $\beta $ have $E$ as a common domain. When we assert $\alpha
\preceq _{E}\beta $ we know that the payoff of $\beta $ is better or equal
to the payoff of $\alpha $ for all conditions $e\in E$. In the case of
expected value, an overall valuation of the performances of $\alpha $ and $%
\beta $ is given by separately calculating the weighted average of payoffs
of each one of them, for all conditions. There is no need of a common set of
conditions $E$. But even if there is such an $E$ , we don't know whether the
payoff of $\beta $ is better or equal to the payoff of $\alpha $ for all
conditions $e\in E$, we know only that it is so for some $e$ and in
particular for $e$ with an high probability value.

Now we introduce a preference relation $\preceq _{v}$ on acts in $A(\mathcal{%
A})$ that borrows from $\preceq _{E}$ the comparison of conditions and from $%
\preceq _{\exp }$the reference to probabilities. For all $E,D\in \Pi (%
\mathcal{A})$, we say that $E$ is a \textit{refinement} of $D$, in symbols $%
E\leq B$ when, for all $e\in E$, there is a $d_{e}\in D$ such that $e\leq
d_{e}$ . Such an element of $D$ is unique (see Appendix \ref{Appendix A},
lemma \ref{lemma2}) and this is why we denote it by $d_{e}$. If $\alpha $, $%
\beta \in $ $A(\mathcal{A})$, where $\alpha :E\rightarrow R$ and $\beta
:D\rightarrow R$, and $v$ is an isotone valuation on $\mathcal{A}$, we say
that $\beta $ \textit{is preferred to} $\alpha $ \textit{with respect to} $%
\nu $, in symbols $\alpha \preceq _{v}\beta $, when the following conditions
are satisfied:

\begin{enumerate}
\item  $E\leq B$,

\item  for all $e\in E$, $\alpha (e)\leq \beta (d_{e})\frac{v(e)}{v(d_{e})}$.
\end{enumerate}

As in the case of $\alpha \preceq _{E}\beta $, the performances of $\alpha $
and $\beta $ are compared with respect to the single conditions of the acts
involved. As in the case of $\alpha \preceq _{\exp }\beta $, the comparison
of $\alpha $ and $\beta $ depends on $v$, i.e. on the probability values of
the relevant conditions. In fact, we cannot compare directly $\alpha (e)$
with $\beta (d_{e})$, as we did with the relation of dominance, because $e$
and $d_{e}$ belong to different sets of conditions, so we compare $\alpha
(e) $ with $\beta (d_{e})\frac{v(e)}{v(d_{e})}$. As $\frac{v(e)}{v(d_{e})}<1$%
, because $e\leq d_{e}$ implies $v(e)\leq v(d_{e})$, we compare $\alpha (e)$
with a reduced $\beta (d_{e})$. \ This reduction can be justified as
follows. The value $\beta (d)$, for any $d\in D$, can be smeared on the set $%
E_{d}=\{x\in E:x\leq d\}$ as the set $\{\beta (d)\frac{v(x)}{v(d)}:x\in
E_{d}\}$. In fact, we have 
\begin{eqnarray*}
\sum \{\beta (d)\frac{v(x)}{v(d)}:x\in E_{d}\} &=&\frac{\beta (d)}{v(d)}\sum
\{v(x):x\in E_{d}\} \\
&=&\frac{\beta (d)}{v(d)}v(\bigvee E_{d}) \\
&=&\frac{\beta (d)}{v(d)}v(d) \\
&=&\beta (d),
\end{eqnarray*}
where the second line follows because $x\wedge x^{\prime }=0$ for all $x$, $%
x^{\prime }\in E$ and the third line because $\bigvee E_{d}=d$ by theorem 
\ref{teo1} of Appendix \ref{Appendix A}. So we can compare $\alpha (e)$ with 
$\beta (d_{e})\frac{v(e)}{v(d_{e})}$ on $E$ and this is just clause 2.

The following theorem shows that $\preceq _{E}$ is the restriction to $A(E)$
of $\preceq _{v}$ defined on $A(\mathcal{A})$.

\begin{theorem}
\label{teo2.1}For all $\alpha $, $\beta \in A(E)$, $\alpha \preceq _{v}\beta 
$ iff $\alpha \preceq _{E}\beta $.
\end{theorem}

\begin{proof}
We have $\alpha \preceq _{v}\beta $ iff, for all $e\in E$, $\alpha (e)\leq
\beta (e_{e})\frac{v(e)}{v(e_{e})}$, where $e_{e}$ is the only $x\in E$ such
that $e\leq x$. But $e_{e}=e$ so $\alpha \preceq _{v}\beta $ iff for all $%
e\in E$, $\alpha (e)\leq \beta (e)\frac{v(e)}{v(e)}$ iff $\alpha (e)\leq
\beta (e)$ iff $\alpha \preceq _{E}\beta $.
\end{proof}

The following theorem shows the relationship between $\preceq _{\exp }$ and $%
\preceq _{v}$.

\begin{theorem}
\label{teo2.2}For all $\alpha $, $\beta \in A(\mathcal{A})$, if $v$ is an
isotone valuation then $\alpha \preceq _{v}\beta $ implies $\alpha \preceq
_{\exp }\beta $.
\end{theorem}

\begin{proof}
We suppose that $\alpha :E\rightarrow R$ and $\beta :D\rightarrow R$. We
have 
\begin{eqnarray*}
\exp (\alpha ,v) &=&\sum \{\alpha (e)v(e):e\in E\} \\
&=&\sum \{\sum \{\alpha (x)v(x):x\in E_{d}\}:d\in D\} \\
&\leq &\sum \{\beta (d)v(d):d\in D\} \\
&=&\exp (\beta ,v),
\end{eqnarray*}
where the second line follows because $\{E_{d}:d\in D\}$ is a set theoretic
partition on $E$ (see lemma \ref{lemma4}of Appendix \ref{Appendix A}) and
the third line because 
\begin{eqnarray*}
\sum \{\alpha (x)v(x):x\in E_{d}\}\leq &&\sum \{\alpha (x)v(d):x\in E_{d}\}
\\
\leq &&\sum \{\beta (d)v(x):x\in E_{d}\} \\
= &&\beta (d)\sum \{v(x):x\in E_{d}\} \\
= &&\beta (d)v(d),
\end{eqnarray*}
where the first line follows because $x\leq d$ implies $v(x)\leq v(d)$, as $%
v $ is isotone, and the second line follows because $\alpha \preceq
_{v}\beta $ by hypothesis. In fact, $\alpha \preceq _{v}\beta $ implies $%
\alpha (x)\leq \beta (d)\frac{v(x)}{v(d)}$, for all $x\in E_{d}$, so $\alpha
(x)v(d)\leq \beta (d)v(x)$. The last line follows because $\sum \{v(x):x\in
E_{d}\}=v(\bigvee E_{d})=v(d)$ (see theorem \ref{teo1} of Appendix \ref
{Appendix A}).
\end{proof}

Of course, we cannot substitute implication with equivalence in the
preceding theorem, because $\exp (a,v)\leq \exp (\beta ,v)$ may hold between
acts $\alpha $ and $\beta $ whose domains $E$ and $D$ are such that $E\nleq
D $ and $D\nleq E$. We observe that, as a consequence of the two preceding
theorems, $\alpha \preceq _{E}\beta $ implies $\alpha \preceq _{\exp }\beta $%
.

\begin{theorem}
\label{teo2.3}$(A(\mathcal{A}),\preceq _{v})$ is a partially ordered set.
\end{theorem}

\begin{proof}
Reflexivity. For all act $\alpha :E\rightarrow R$, we have: 1) $E\leq E$,
because $\leq $ is a partial order on the set of all partitions on $\mathcal{%
A}$ (see theorem \ref{teopartlatt} of Appendix \ref{Appendix A}); 2) for all 
$e\in E$, $\alpha (e)\leq \alpha (e_{e})\frac{v(e)}{v(e_{e})}$ because, for
all $e\in E$, $e_{e}=e$ (the only $x\in E$ such that $e\leq x$ being $e$
itself), so $\frac{v(e)}{v(e_{e})}=1$. This proves that $\alpha \preceq
_{v}\alpha $.

Transitivity. We assume that $\alpha \preceq _{v}\beta $ and $\beta \preceq
_{v}\gamma $, where $\alpha :E\rightarrow R$, $\beta :D\rightarrow R$ and $%
\gamma :G\rightarrow R$. 1) By hypothesis, $E\leq D$ and $D\leq G$, so $%
E\leq G$, because $\leq $ is a partial order on partitions. 2) By
hypothesis, for all $e\in E$, we have $\alpha (e)\leq \beta (d_{e})\frac{v(e)%
}{v(d_{e})}$ and for all $d\in D$ we have $\beta (d)\leq \gamma (g_{d})\frac{%
v(d)}{v(g_{d})}$: in particular, $\beta (d_{e})\leq \gamma (g_{d_{e}})\frac{%
v(d_{e})}{v(g_{d_{e}})}$. By definition, $e\leq d_{e}\leq g_{d_{e}}$ and $%
e\leq g_{e}$, so $g_{d_{e}}=g_{e}$, because there is only one $x\in G$ such
that $e\leq x$, so $\beta (d_{e})\leq \gamma (g_{e})\frac{v(d_{e})}{v(g_{e})}
$. Then we have 
\begin{equation*}
\alpha (e)\leq \beta (d_{e})\frac{v(e)}{v(d_{e})}\leq \gamma (g_{e})\frac{%
v(d_{e})}{v(g_{e})}\frac{v(e)}{v(d_{e})}=\gamma (g_{e})\frac{v(e)}{v(g_{e})}.
\end{equation*}

This proves that $\alpha \preceq _{v}\gamma $.

Antisymmetry. We assume that $\alpha \preceq _{v}\beta $ and $\beta \preceq
_{v}\alpha $. 1) By hypothesis, $E\leq D$ and $D\leq E$, so $E=D$, because $%
\leq $ is a partial order on partitions. 2) By hypothesis, for all $e\in E$,
we have $\alpha (e)\leq \beta (d_{e})\frac{v(e)}{v(d_{e})}$ and for all $%
d\in D$ we have $\beta (d)\leq \alpha (e_{d})\frac{v(d)}{v(e_{d})}$. As $E=D$%
, $\alpha (e)\leq \beta (e)\frac{v(e)}{v(e)}\ =\beta (e)$. For the same
reason, we have $\beta (e)\leq \alpha (e)\frac{v(e)}{v(e)}=\alpha (e)$, so $%
\alpha (e)=\beta (e)$. This proves that $\alpha =\beta $.
\end{proof}

\section{The lattice of acts}

Given an act $\beta :D\rightarrow R$ and a partition $E$ such that $E\leq D$%
, we can downgrade $\beta $ to an act $\beta _{E}:E\rightarrow R$ setting,
for all $e\in E$%
\begin{equation*}
\beta _{E}(e)=\beta (d_{e})\frac{v(e)}{v(d_{e})}.
\end{equation*}

We have $\beta _{E}\preceq _{v}\beta $ by definition of $\beta _{E}$. In
fact, $\beta _{E}$ is the best approximation from below to $\beta $ in $A(E)$%
, as shown in corollary \ref{cor4.1}.

\begin{lemma}
$(A(\mathcal{A}),\preceq _{v})$ is closed with respect to $\inf $.
\end{lemma}

\begin{proof}
We know that $\Pi (A)$ is a lattice, by theorem \ref{teopartlatt} of
Appendix \ref{Appendix A}. Given $\alpha :E\rightarrow R$ and $\beta
:D\rightarrow R$, we define $\phi :E\wedge D\rightarrow R$ setting, for all $%
z\in E\wedge D$, 
\begin{equation*}
\phi (z)=\min (\alpha _{E\wedge D}(z),\beta _{E\wedge D}(z)).
\end{equation*}
We show that $\phi =\inf (\alpha ,\beta )$.

1. $\phi \preceq _{v}\alpha $, $\beta $. By definition, we have $E\wedge
D\leq E$. Then we have, for all $z\in E\wedge D$, 
\begin{equation*}
\phi (z)=\min (\alpha (e_{z})\frac{v(z)}{v(e_{z})},\beta (d_{z})\frac{v(z)}{%
v(d_{z})})\leq \alpha (e_{z})\frac{v(z)}{v(e_{z})},
\end{equation*}
so $\phi \preceq _{v}\alpha $. In the same way, we can prove that $\phi
\preceq _{v}\beta $

2. We prove that, for all $\gamma :G\rightarrow R$, if $\gamma \preceq
_{v}\alpha $ and $\gamma \preceq _{v}\beta $, then $\gamma \preceq _{v}\phi $%
. By hypothesis $G\leq E$ and $G\leq D$, so $G\leq E\wedge D$. We have to
show that, for all $g\in G$, $\gamma (g)\leq \phi (z_{g})\frac{v(g)}{v(z_{g})%
}$ where $z_{g}$ is the element of $E\wedge D$ such that $g\leq z_{g}$. By
definition of $\phi $, this amounts to prove that 
\begin{eqnarray*}
\gamma (g) &\leq &\min (\alpha _{E\wedge D}(z_{g}),\beta _{E\wedge D}(z_{g}))%
\frac{v(g)}{v(z_{g})} \\
&=&\min (\alpha (e_{z_{g}})\frac{v(z_{g})}{v(e_{z_{g}})},\beta (d_{z_{g}})%
\frac{v(z_{g})}{v(d_{z_{g}})})\frac{v(g)}{v(z_{g})} \\
&=&\min (\frac{\alpha (e_{z_{g}})}{v(e_{z_{g}})},\frac{\beta (d_{z_{g}})}{%
v(d_{z_{g}})})v(g),
\end{eqnarray*}
where $z_{g}\leq e_{z_{g}}$ and $z_{g}\leq d_{z_{g}}$. By hypothesis, for
all $g\in G$, we have $\gamma (g)\leq \alpha (e_{g})\frac{v(g)}{v(e_{g})}$
and $\gamma (g)\leq \beta (d_{g})\frac{v(g)}{v(d_{g})}$, so 
\begin{equation*}
\gamma (g)\leq \min (\frac{\alpha (e_{g})}{v(e_{g})},\frac{\beta (d_{g})}{%
v(d_{g})})v(g).
\end{equation*}
We conclude the proof by observing that, from $g\leq e_{g}$ and $g\leq
z_{g}\leq e_{z_{g}}$, we can derive $e_{g}=e_{z_{g}}$, because $G$ and $E$
are partitions. In the same way, from $g\leq d_{g}$ and $g\leq z_{g}\leq
d_{z_{g}},$ we have $d_{g}=d_{z_{g}}$. This proves that $\phi $ is the
greatest lower bound of $\alpha $ and $\beta $.
\end{proof}

Given an act $\beta :D\rightarrow R$ and a partition $G$ such that $D\leq G$%
, we can upgrade $\beta $ to an act $\beta ^{G}:G\rightarrow R$ setting, for
all $g\in G$, 
\begin{equation*}
\beta ^{G}(g)=\max (\beta (x)\frac{v(g)}{v(x)}:x\in D_{g}).
\end{equation*}

We have $\beta \preceq _{v}\beta ^{G}$. In fact $d\in D_{g_{d}}$ so $\beta
(d)\frac{v(g)}{v(d)}\leq \beta ^{G}(g)$ and then $\beta (d)\leq \beta ^{G}(g)%
\frac{v(d)}{v(g)}$, thus proving that $\beta \preceq _{v}\beta ^{G}$. From
corollary \ref{cor4.1} we can see that $\beta ^{G}$ is the best
approximation from above to $\beta $ in $A(G)$.

\begin{lemma}
$(A(\mathcal{A}),\preceq _{v})$ is closed with respect to $\sup $.
\end{lemma}

\begin{proof}
Given $\alpha :E\rightarrow R$ and $\beta :D\rightarrow R$, we define $\phi
:E\vee D\rightarrow R$ setting, for all $z\in E\vee D$, 
\begin{equation*}
\phi (z)=\max (\alpha ^{E\vee D}(z),\beta ^{E\vee D}(z)).
\end{equation*}
We show that $\phi =\sup (\alpha ,\beta )$.

1. $\alpha $, $\beta \preceq _{v}\phi $. In order to show that $\alpha
\preceq _{v}\phi $, we must prove that, for all $e\in E$, $\alpha (e)\leq
\phi (z_{e})\frac{v(e)}{v(z_{e})}.$ We have 
\begin{eqnarray*}
\phi (z_{e}) &=&\max (\alpha ^{E\vee D}(z_{e}),\beta ^{E\vee D}(z_{e})) \\
&=&\max (v(z_{e})\max (\frac{\alpha (x)}{v(x)}:x\in E_{z_{e}}),v(z_{e})\max (%
\frac{\beta (x)}{v(x)}:x\in D_{z_{e}})) \\
&\geq &v(z_{e})\max (\frac{\alpha (x)}{v(x)}:x\in E_{z_{e}}),
\end{eqnarray*}
so 
\begin{eqnarray*}
\phi (z_{e})\frac{v(e)}{v(z_{e})} &\geq &v(z_{e})\max (\frac{\alpha (x)}{v(x)%
}:x\in E_{z_{e}})\frac{v(e)}{v(z_{e})} \\
&=&v(e)\max (\frac{\alpha (x)}{v(x)}:x\in E_{z_{e}}) \\
&\geq &v(e)\frac{\alpha (e)}{v(e)} \\
&=&\alpha (e),
\end{eqnarray*}
where the third line follows because $e\in E_{z_{e}}$ and then $\max (\frac{%
\alpha (x)}{v(x)}:x\in E_{z_{e}})\geq \frac{\alpha (e)}{v(e)}$. In the same
way we can prove that $\beta \preceq _{v}\phi $.

2. For all $\gamma :G\rightarrow R$, if $\alpha \preceq _{v}\gamma $ and $%
\beta \preceq _{v}\gamma $, then $\phi \preceq _{v}\gamma $. By hypothesis $%
E\leq G$ and $D\leq G$, so $E\vee D\leq G$. We have to show that, for all $%
z\in E\vee D$, $\phi (z)\leq \gamma (g_{z})\frac{v(z)}{v(g_{z})}$. By
hypothesis we have, for all $e\in E$, $\alpha (e)\leq \gamma (g_{e})\frac{%
v(e)}{v(g_{e})}$ so $\alpha (e)\frac{v(g_{e})}{v(e)}\leq \gamma (g_{e})$. In
particular, for all $x\in E_{z}$ we have $\alpha (x)\frac{v(g_{x})}{v(x)}%
\leq \gamma (g_{x})$. When $x\in E_{z}$ we have also $x\leq z\leq g_{z}$ and 
$x\leq g_{x}$, so $g_{z}=g_{x}$ because $G$ and $E$ are partitions. Then we
can conclude that, for all $x\in E_{z}$, $\alpha (x)\frac{v(g_{z})}{v(x)}%
\leq \gamma (g_{z})$ and so 
\begin{equation*}
\max (\alpha (x)\frac{v(g_{z})}{v(x)}:x\in E_{z})\leq \gamma (g_{z})
\end{equation*}
and 
\begin{equation*}
v(z)\max (\frac{\alpha (x)}{v(x)}:x\in E_{z})\leq \gamma (g_{z})\frac{v(z)}{%
v(g_{z})}.
\end{equation*}
In the same way, from the hypothesis for all $d\in D$, $\beta (d)\leq \gamma
(g_{d})\frac{v(d)}{v(g_{d})}$, we can prove that 
\begin{equation*}
v(z)\max (\frac{\beta (x)}{v(x)}:x\in D_{z})\leq \gamma (g_{z})\frac{v(z)}{%
v(g_{z})},
\end{equation*}
so we can conclude that 
\begin{eqnarray*}
\gamma (g_{z})\frac{v(z)}{v(g_{z})} &\geq &\max (v(z)\max (\frac{\alpha (x)}{%
v(x)}:x\in E_{z}),v(z)\max (\frac{\beta (x)}{v(x)}:x\in D_{z})) \\
&=&\max (\alpha ^{E\vee D}(z),\beta ^{E\vee D}(z)) \\
&=&\phi (z).
\end{eqnarray*}
\end{proof}

\begin{theorem}
\label{teo4}$(A(\mathcal{A}),\preceq _{v})$ is a lattice with $(A(E),\preceq
_{E})$ as a sublattice.
\end{theorem}

\begin{proof}
$(A(\mathcal{A}),\preceq _{v})$ is a lattice by the preceding lemmas. We
prove that $\inf $ $(A(E),\preceq _{E})$ is a sublattice of $(A(\mathcal{A}%
),\preceq _{v})$. In the first place, we show that for all $\alpha $, $\beta
\in A(E)$, $\inf_{E}(\alpha ,\beta )=\inf_{v}(\alpha ,\beta )$, where $%
\inf_{E}$ denotes $\inf $ in $A(E)$ and $\inf_{v}$ denotes $\inf $ in $A(%
\mathcal{A})$. We observe that $\inf_{E}(\alpha ,\beta )$ is a lower bound
of $\{\alpha ,\beta \}$ in $(A(\mathcal{A}),\preceq _{v})$ because $%
\inf_{E}(\alpha ,\beta )\preceq _{E}\alpha $, $\beta $ implies $%
\inf_{E}(\alpha ,\beta )\preceq _{v}\alpha $, $\beta $, by theorem \ref
{teo2.1}. We can see that $\inf_{E}(\alpha ,\beta )$ is the greatest lower
bound of $\{\alpha ,\beta \}$ in $(A(\mathcal{A}),\preceq _{v})$ as follows:
we suppose that $\xi \preceq _{v}\alpha $, $\beta $, where $\xi \in A(%
\mathcal{A})$ is an act $\xi :G\rightarrow R$ where $G\leq E$, and we show
that $\xi \preceq _{v}\inf_{E}(\alpha ,\beta )$. So we must show that, for
all $g\in G$, $\xi (g)\leq \inf_{E}(\alpha ,\beta )(e_{g})\frac{v(g)}{%
v(e_{g})}=\min (\alpha (e_{g}),\beta (e_{g}))\frac{v(g)}{v(e_{g})}$. (We
remember that $\alpha $ and $\beta $ are functions $E\rightarrow R$ and $%
\inf_{E}$ is defined pointwise.) By hypothesis, $\xi (g)\leq \alpha (e_{g})%
\frac{v(g)}{v(e_{g})}$ and $\xi (g)\leq \beta (e_{g})\frac{v(g)}{v(e_{g})}$,
so 
\begin{eqnarray*}
\xi (g) &\leq &\min (\alpha (e_{g})\frac{v(g)}{v(e_{g})},\beta (e_{g})\frac{%
v(g)}{v(e_{g})}) \\
&=&\min (\alpha (e_{g}),\beta (e_{g}))\frac{v(g)}{v(e_{g})},
\end{eqnarray*}
where the second line follows because $\min (ax,bx)=\min (a,b)x$ when $x\geq
0$. Finally, we show that $\sup_{E}(\alpha ,\beta )=\sup_{v}(\alpha ,\beta )$%
. On one side, $\sup_{E}(\alpha ,\beta )$ is an upper bound of $\{\alpha
,\beta \}$ in $(A(\mathcal{A}),\preceq _{v})$, because $\alpha $, $\beta
\preceq _{E}\sup_{E}(\alpha ,\beta )$ and so $\alpha $, $\beta \preceq
_{v}\sup_{E}(\alpha ,\beta )$ by theorem \ref{teo2.1}. On the other side,
for all $\xi \in A(\mathcal{A})$ such that $\alpha $, $\beta \preceq _{v}\xi 
$, we can show that $\sup_{E}(\alpha ,\beta )$ $\preceq _{v}\xi $. In fact, $%
\xi :G\rightarrow R$ for some $G\geq E$, so by hypothesis we have $\alpha
(e)\leq \xi (g_{e})\frac{v(e)}{v(g_{e})}$ and $\beta (e)\leq \xi (g_{e})%
\frac{v(e)}{v(g_{e})}$. Then 
\begin{equation*}
\sup_{E}(\alpha ,\beta )(e)=\max (\alpha (e),\beta (e))\leq \xi (g_{e})\frac{%
v(e)}{v(g_{e})}.
\end{equation*}
thus proving that $\sup_{E}(\alpha ,\beta )$ $\preceq _{v}\xi $.
\end{proof}

\begin{corollary}
\hfill

\begin{enumerate}
\item  If $\beta :D\rightarrow R$ and $E\leq D$, then $\beta _{E}=\bigvee
\{\xi \in A(E):\xi \preceq _{v}\beta \}$, where $\bigvee $ is taken in $A(%
\mathcal{A})$.

\item  If $\beta :D\rightarrow R$ and $D\leq G$, then $\beta ^{G}=\bigwedge
\{\xi \in A(G):\beta \preceq _{v}\xi \}$, where $\bigwedge $ is taken in $A(%
\mathcal{A})$. \label{cor4.1}
\end{enumerate}
\end{corollary}

\begin{proof}
1. On one side, we show that $\ \beta _{E}$ is an upper bound of $\{\xi \in
A(E):\xi \preceq _{v}\beta \}$. If $\xi :E\rightarrow R$ is such that $\xi
\preceq _{v}\beta $ then, for all $e\in E$, $\xi (e)\leq \beta (d_{e})\frac{%
v(e)}{v(d_{e})}=\beta _{E}(e)$. So $\xi \preceq _{E}\beta _{E}$ and then\ $%
\xi \preceq _{v}\beta _{E}$ by theorem \ref{teo2.1}. On the other side, let $%
\delta :E\rightarrow R$ be an upper bound of $\{\xi \in A(E):\xi \preceq
_{v}\beta \}$, then $\beta _{E}\preceq _{v}\delta $ because $\beta
_{E}\preceq _{v}\beta $, by definition of $\beta _{E}$.

2. On one side, we show that $\ \beta ^{G}$ is a lower bound of $\{\xi \in
A(G):\beta \preceq _{v}\xi \}$. If $\xi :G\rightarrow R$ is such that $\beta
\preceq _{v}\xi $ then, for all $d\in D$, $\beta (d)\leq \xi (g_{d})\frac{%
v(d)}{v(g_{d})}$ and $\beta (d)\frac{v(g_{d})}{v(d)}\leq \xi (g_{d})$. We
observe that, for all $g\in G$, $g=g_{d}$ holds for all $d\in D_{g}$, so $%
\beta (d)\frac{v(g)}{v(d)}\leq \xi (g)$ holds for all $d\in D_{g}$ and then 
\begin{equation*}
\max (\beta (x)\frac{v(g)}{v(x)}:x\in D_{g})\leq \xi (g).
\end{equation*}
In this way we have shown that, for all $g\in G$, $\beta ^{G}(g)\leq \xi (g)$
and so $\beta ^{G}\preceq _{G}\xi $. By theorem \ref{teo2.1} we can conclude
that $\beta ^{G}\preceq _{v}\xi $. \ On the other side, let $\delta
:G\rightarrow R$ be a lower bound of $\{\xi \in A(G):\beta \preceq _{v}\xi
\} $, \ then $\delta \preceq _{v}\beta ^{G}$ because $\beta \preceq
_{v}\beta ^{G}$, as we have shown before the preceding lemma.
\end{proof}

In the preceding corollary, $\bigvee $ can indifferently be taken in $A(E)$
and $\bigwedge $ in $A(G)$.

\begin{corollary}
If $\alpha :E\rightarrow R$ and $\beta :D\rightarrow R$ then $\inf (\alpha
,\beta )=\inf (\alpha _{E\wedge D},\beta _{E\wedge D})$ and $\sup (\alpha
,\beta )=\sup (\alpha ^{E\vee D},\beta ^{E\vee D})$.
\end{corollary}

\begin{proof}
We denote with $\inf $ the greatest lower bound taken in $(A(\mathcal{A}%
),\preceq _{v})$ and with $\inf_{E\wedge D}$ the greatest lower bound taken
in $(A(E),\preceq _{E\wedge D})$. As $\alpha _{E\wedge D}$ and $\beta
_{E\wedge D}$ are acts in $A(E\wedge D)$, for all $z\in E\wedge D$ we have $%
\inf_{E\wedge D}(\alpha _{E\wedge D},\beta _{E\wedge D})(z)=\min (\alpha
_{E\wedge D}(z),\beta _{E\wedge D}(z))$ because $\inf_{E\wedge D}$ is
defined pointwise. But $\inf (\alpha ,\beta )(z)=\min (\alpha _{E\wedge
D}(z),\beta _{E\wedge D}(z))$ by definition, so $\inf (\alpha ,\beta )=\inf
(\alpha _{E\wedge D},\beta _{E\wedge D})$. Finally, we have $\inf_{E\wedge
D}(\alpha _{E\wedge D},\beta _{E\wedge D})=\inf (\alpha _{E\wedge D},\beta
_{E\wedge D})$ by theorem \ref{teo4}. The same kind of proof works for the
least upper bound.
\end{proof}

\section{The comparison of acts}

We can summarize the different ways of comparing acts introduced so far as
follows. Given $\alpha :E\rightarrow R$ and $\beta :D\rightarrow R$, we can
always compare $\alpha $ with $\beta $ in $\preceq _{\exp }$. If $\alpha $
and $\beta $ have the same domain, $E=D$, they can also be compared in $%
\preceq _{E}$. We know that $\alpha \preceq _{E}\beta $ implies $\alpha
\preceq _{\exp }\beta $. (We have $\alpha \preceq _{E}\beta $ implies $%
\alpha \preceq _{v}\beta $ by theorem \ref{teo2.1} and $\alpha \preceq
_{v}\beta $ implies $\alpha \preceq _{\exp }\beta $ by theorem \ref{teo2.2}%
.) If $\alpha $ and $\beta $ have different, but comparable, domains, i.e. $%
E\leq D$ or $D\leq E$, then we can compare $\alpha $ with $\beta $ in $%
\preceq _{v}$. We know that $\alpha \preceq _{v}\beta $ implies $\alpha
\preceq _{\exp }\beta $ by theorem \ref{teo2.2}. If $E$ and $D$ are
incomparable, we can resort to the best approximations from below to $\alpha 
$ and $\beta $ in $A(E\wedge D)$. So we define a preference relation $\alpha
\vartriangleleft \beta $ iff $\alpha _{E\wedge D}\leq _{E\wedge D}\beta
_{E\wedge D}$. An intuitive meaning may be attached to $\vartriangleleft $
if we observe how the set of conditions $E\wedge D$ arises from $E$ and $D$
by meet. As an example, we set $E=\{\uparrow \$,\downarrow \$\}$ and $%
D=\{\uparrow \pounds ,\downarrow \pounds \}$, where $\uparrow $ means
`rises' and $\downarrow $ means `sinks'. Then $E\wedge D=\{\uparrow
\$\&\uparrow \pounds ,\uparrow \$\&\downarrow \pounds ,\downarrow
\$\&\uparrow \pounds ,\downarrow \$\&\downarrow \pounds \}$ is a natural
common set of conditions for $\alpha _{E\wedge D}$ and $\beta _{E\wedge D}$
where every condition $e\in E$ splits in the different cases $\{e\wedge
d:d\in D\}$.

We can easily prove that $\vartriangleleft $ is reflexive transitive and
antisymmetric. Firstly, we observe that $\alpha _{E\wedge D}\leq _{E\wedge
D}\beta _{E\wedge D}$ iff for all $e\wedge d$ in $E\wedge D$, $\frac{\alpha
(e)}{v(e)}\leq \frac{\beta (d)}{v(d)}$. In fact, 
\begin{equation*}
\alpha _{E\wedge D}(e\wedge d)\leq \beta _{E\wedge D}(e\wedge d)\text{ iff }%
\alpha (e)\frac{v(e\wedge d)}{v(e)}\leq \beta (e)\frac{v(e\wedge d)}{v(d)}%
\text{ iff }\frac{\alpha (e)}{v(e)}\leq \frac{\beta (d)}{v(d)}.
\end{equation*}
Now we can easily see that $\vartriangleleft $ is a partial order. The
following theorem shows that $\vartriangleleft $ can be seen as a
generalization of $\leq _{E}$ and $\leq _{v}$.

\begin{lemma}
If $\alpha :E\rightarrow R$ then $\alpha _{E}=\alpha $ and $\alpha
^{E}=\alpha $.
\end{lemma}

\begin{proof}
For all $e\in E$, $\alpha _{E}(e_{e})=\alpha (e)\frac{v(e)}{v(e_{e})}=\alpha
(e)$, because $e_{e}=e$. For all $e\in E$, $\alpha ^{B}(e)=\max \{\alpha (x)%
\frac{v(e)}{v(x)}:x\in E_{e}\}=\alpha (e)$, because $E_{e}=\{e\}$.
\end{proof}

\begin{theorem}
If $\alpha :E\rightarrow R$ and $\beta :D\rightarrow R$ then

\begin{enumerate}
\item  $E=D$ implies $\alpha \vartriangleleft \beta $ iff $\alpha \leq
_{E}\beta $,

\item  $\alpha \leq _{v}\beta $ implies $\alpha \vartriangleleft \beta $; $%
\alpha \vartriangleleft \beta $ and $E\leq D$ imply $\alpha \leq _{v}\beta $.
\end{enumerate}
\end{theorem}

\begin{proof}
1. If $E=D$ then $E\wedge D=E$ and so $\alpha \vartriangleleft \beta $ iff $%
\alpha _{E}\ \leq _{E}\beta _{E}$ iff $\alpha \leq _{E}\beta $, as $\alpha
_{E}=\alpha $ and $\beta _{E}=\beta _{D}=\beta $, by the lemma.

2. We must show that $\alpha _{E\wedge D}\leq _{E\wedge D}\beta _{E\wedge D}$%
. By our hypothesis $\alpha \leq _{v}\beta $, $E\leq D$ holds, so we can
reduce ourselves to prove that $\alpha _{E}\leq _{E}\beta _{E}$ i.e. $\alpha
\leq _{E}\beta _{E}$. So we have to prove that $\alpha (e)\leq \beta
_{E}(e)=\beta (d_{e})\frac{v(e)}{v(d_{e})}$, what follows from our
hypothesis. Now we assume $\alpha \vartriangleleft \beta $ and $E\leq D$,
then $\alpha _{E\wedge D}\leq _{E\wedge D}\beta _{E\wedge D}$ and $\alpha
\leq _{E}\beta _{E}$, so for all $e\in E$, $\alpha (a)\leq \beta (d_{e})%
\frac{v(e)}{v(d_{e})}$ and $\alpha \leq _{v}\beta $ follows.
\end{proof}

Dually, we can define a preference relation setting $\alpha
\blacktriangleleft \beta $ iff $\alpha ^{E\vee D}\leq _{E\vee D}\beta
^{E\vee D}$. Now $\alpha ^{E\vee D}$ and $\beta ^{E\vee D}$ are the best
approximation to $\alpha $ and $\beta $ from above in $A(E\vee D)$. The join
of the set of conditions $E$ and $D$ is trivial in the example above. In
general, if $|E|=|D|=2$, then $E=\{e,\lnot e\}$ and $D=\{d,\lnot d\}$, where 
$\lnot $ denotes the complement operation, so $E\vee D=1$, the top element
of the lattice $\Pi (\mathcal{A})$. We can give a non-trivial example of $%
E\vee D$ as follows. We consider the interval $[0,1)$ as the price range of
a good and we define five subintervals $a=[0,0.2)$, $b=[0.2,0.4)$, $%
c=[0.4,0.6)$, $d=[0.6,0.8)$, $e=[0.8,1)$. Let $\mathcal{A}=\mathcal{P}(A)$,
where $A=\{a,b,c,d,e\}$. We identify each $x\in A$ with the singleton $\{x\}$
and denote with $a|b|c|d|e$ the least partition in $\Pi (\mathcal{A})$.
There are 52 partitions in $\Pi (\mathcal{A})$, but we can fix our attention
on the four-elements lattice of the following figure, where juxtaposition
denotes set-union (i.e. $abc=a\cup b\cup c$).

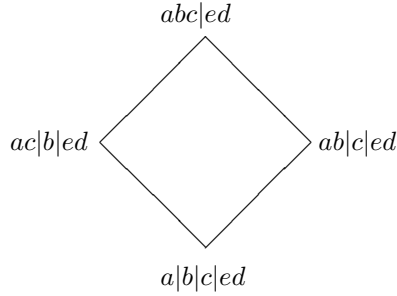
\begin{center}
\setlength{\unitlength}{1mm}

\begin{picture}(45,45)
\put(19,5){\line(1,1){14}}
\put(5,19){\line(1,1){14}}

\put(5,19){\line(1,-1){14}}
\put(19,33){\line(1,-1){14}}

\put(-7,18){$ac|b|ed$}
\put(13,0){$a|b|c|ed$}
\put(34,18){$ab|c|ed$}
\put(13,35){$abc|ed$}


\end{picture}
\end{center}

\noindent We set $E=ac|b|ed$ and $D=ab|c|ed$. Let $\alpha :E\rightarrow R$
and $\beta :D\rightarrow R$ be two acts. As we suppose that $[0,1)$ be the
price range of a good, then $\alpha (b)$ represents the payoff of $\alpha $
when the price is in $[0.2,0.4)$. The same holds for $\beta $. Now $E\wedge
D $ is the set of conditions containing all non-empty meets of conditions in 
$E $ with conditions in $D$ and $E\vee D$ is the set of all (minimal)
common-joins from $E$ and $D$ (joins of conditions of $\alpha $ that arise
also as joins of conditions of $\beta $).

We can easily prove that $\blacktriangleleft $ is reflexive and
antisymmetric. Firstly, we observe that $\alpha ^{E\vee D}\leq _{E\vee
D}\beta ^{E\vee D}$ iff for all $w\in E\vee D$, $\alpha ^{E\vee D}(w)\leq
\beta ^{E\vee D}(w)$ and this happens iff $\max \{\frac{\alpha (x)}{v(x)}%
:x\in E_{w}\}\leq \max \{\frac{\beta (x)}{v(x)}:x\in D_{w}\}$. Transitivity
fails, as can be easily seen by a counterexemple.

The following theorem shows that $\blacktriangleleft $ can be seen as a
generalization of $\leq _{E}$ and $\leq _{v}$.

\begin{theorem}
If $\alpha :E\rightarrow R$ and $\beta :D\rightarrow R$ then

\begin{enumerate}
\item  $E=D$ implies $\alpha \blacktriangleleft \beta $ iff $\alpha \leq
_{E}\beta $,

\item  $\alpha \leq _{v}\beta $ implies $\alpha \vartriangleleft \beta $; $%
\alpha \blacktriangleleft \beta $ and $E\leq D$ imply $\alpha \leq _{v}\beta 
$.
\end{enumerate}
\end{theorem}

\begin{proof}
1. If $E=D$ then $E\vee D=E$ so $\alpha \blacktriangleleft \beta $ iff $%
\alpha ^{E}\leq _{E}\beta ^{E}$ iff $\alpha \leq _{E}\beta $, the the lemma
above.

2. We assume $\alpha \leq _{v}\beta $ so $E\leq D$ and, for all $e\in E$, $%
\alpha (a)\leq \beta (d_{e})\frac{v(e)}{v(d_{e})}$ and so $\alpha (a)\frac{%
v(d_{e})}{v(e)}\leq \beta (d_{e})$. We must show that $\alpha ^{E\vee D}\leq
_{E\vee D}\beta ^{E\vee D}$, i.e. $\alpha ^{D}\leq _{D}\beta $. So we can
reduce ourselves to prove that, for all $d\in D$, $\alpha ^{D}(d)\leq \beta
(d)$, i.e. $\max \{\alpha (x)\frac{v(d)}{v(x)}:x\in E_{d}\}\leq \beta (d)$.
If $x\in E_{d}$ then $d_{x}=d$, so by our hypothesis we have $\alpha (x)%
\frac{v(d)}{v(x)}\leq \beta (d)$ for all $d\in D$ and the result follows. We
assume $\alpha \blacktriangleleft \beta $ and $E\leq D$, then $\alpha
^{E\vee D}\leq _{E\vee D}\beta ^{E\vee D}$, i.e. $\alpha ^{D}\leq _{D}\beta $
and so $\max \{\alpha (x)\frac{v(d)}{v(x)}:x\in E_{d}\}\leq \beta (d)$ for
all $d\in D$. If $x\in E_{d}$ then $\alpha (x)\frac{v(d)}{v(x)}\leq \beta
(d) $ and $\alpha (x)\leq \beta (d)\frac{v(x)}{v(d)}$. But for all $e\in E$,
we have $e\in E_{d_{e}}$, so $\alpha (e)\leq \beta (d_{e})\frac{v(e)}{%
v(d_{e})}$ and $\alpha \leq _{v}\beta $ follows.

\appendix
\end{proof}

\section{Partitions in finite distributive lattices\label{Appendix A}}

The set theoretic notion of partition, introduced in the first paragraph,
can be generalized as follows. Let $\mathcal{A}$ be a finite distributive
lattice, we say that $E\subseteq A$ is a \textit{partition} of $\mathcal{A}$
if

\begin{enumerate}
\item  $\bigvee E=1$,

\item  $e_{2}\wedge e_{2}=0$, for all $e_{2},e_{2}\in E$ with $e_{2}\neq
e_{2}$,

\item  $e\neq 0$, for all $e\in E$.
\end{enumerate}

We denote with $\Pi (A)$ the set of all partitions of $\mathcal{A}$. Of
course, every set theoretic partition $\{X_{i}:i\in I\}$ is also an
algebraic partition of $\mathcal{P}(X)$, the Boolean algebra of all subsets
of $X$. In the following we speak generically of partitions, leaving to the
context to decide whether algebraic or set theoretical partitions are
involved. In general, we speak of a partition \textit{on} (a set) $X$ when a
set theoretical partition is intended, and speak of a partition \textit{of}
(a lattice) $\mathcal{A}$ when an algebraic partition is intended.

\begin{lemma}
\label{lemma1}For all partition $E$ of $\mathcal{A}$ and all $e\in E$, $%
E-\{e\}$ is not a partition of $\mathcal{A}$
\end{lemma}

\begin{proof}
We set $E^{\prime }=$.$E-\{e\}$ and suppose that $E^{\prime }$ is a
partition, then $\bigvee E^{\prime }=1$ and so 
\begin{equation*}
e=e\wedge \bigvee E^{\prime }=\bigvee \{e\wedge e^{\prime }:e^{\prime }\in
E^{\prime }\}=0,
\end{equation*}
because $e$, $e^{\prime }\in E$ and $e\neq e^{\prime }$. But $e\neq 0$,
because $E$ is a partition.
\end{proof}

We define a relation on partitions setting $E\leq D$ iff for all $e\in E$
there is a $d\in D$ such that $e\leq d$. In this case, we say that $E$ is a 
\textit{refinement} of\textit{\ }(or is \textit{finer} than) $D$. The
following lemma shows that there is only one $d$ of this kind, so we can
speak of \textit{the} $d\in D$ such that $e\leq d$ and denote it with $d_{e}$%
.

\begin{lemma}
\label{lemma2}If $E\leq D$ then, for all $e\in E,$ there is only one $d\in D$
such that $e\leq d$.
\end{lemma}

\begin{proof}
We suppose that, for some $e\in E$, there are $d$ and $d^{\prime }$ in $D$
such that $d\neq d^{\prime }$ and $e\leq d$, $d^{\prime }$. Then $e\leq
d\wedge d^{\prime }=0$, but this is absurd because $E$ is a partition.
\end{proof}

The following lemma shows that $d_{x}$, as a function $E\rightarrow D$, is
surjective.

\begin{lemma}
\label{lemma3}If $E\leq D$ then, for all $d\in D$, there is $e\in E$ such
that $e\leq d$.
\end{lemma}

\begin{proof}
As $E\leq D$, for all $e\in E$ there is a $d\in D$ such that $e\leq d$. We
suppose that there is $\overline{d}\in D$ such that, for all $e\in E$, $%
e\nleq \overline{d}$, then we have $1=\bigvee E\leq \bigvee ($ $D-\{%
\overline{d}\})$. Then for all $x$, $y\in D-\{\overline{d}\}$, we have $%
x\wedge y=0$. Finally, for all $x\in D-\{\overline{d}\}$, we have $x\neq 0$.
So $D-\{\overline{d}\}$ is a partition of $\mathcal{A}$, but this is absurd
by lemma \ref{lemma1}.
\end{proof}

If $E\leq D$ then, for all $d\in D$, we define $E_{d}=\{x\in E:x\leq d\}$:
we shall prove that $d=\bigvee E_{d}$.

\begin{lemma}
\label{lemma4}If $E\leq D$ then:

\begin{enumerate}
\item  $\{E_{x}:x\in D\}$ is a (set theoretic) partition on $E$,

\item  $\{\bigvee E_{x}:x\in D\}$ is an (algebraic) partition of $\mathcal{A}
$.
\end{enumerate}
\end{lemma}

\begin{proof}
1. Firstly, we prove that $\bigcup \{E_{x}:x\in D\}=E$. On one side we have $%
\bigcup \{E_{x}:x\in D\}\subseteq E$, because $E_{x}\subseteq E$ for all $%
x\in D$. On the other side, $E\subseteq $ $\bigcup \{E_{x}:x\in D\}$
because, for all $e\in E$, there is a $x\in D$ such that $e\leq x$ and $e\in
E_{x}$. Then we prove that, for all $x$, $x^{\prime }\in D$, $x\neq
x^{\prime }$, we have $E_{x}\cap E_{x^{\prime }}=\emptyset $, because $y\in
E_{x}\cap E_{x^{\prime }}$ implies $y\leq x$ and $y\leq x^{\prime }$ so that 
$y\leq x\wedge x^{\prime }=0$, that is absurd. Finally, we have $E_{x}\neq
\emptyset $, for all $x\in D$, by lemma \ref{lemma3}

2. Firstly, we have 
\begin{equation*}
\bigvee \{\bigvee E_{x}:x\in D\}=\bigvee \bigcup \{E_{x}:x\in D\}=\bigvee E=1
\end{equation*}
because, by point 1), we have $\bigcup \{E_{x}:x\in D\}=E$. Then we have,
for all $x$, $y\in D$, $x\neq y$, 
\begin{equation*}
(\bigvee E_{x})\wedge (\bigvee E_{y})=\bigvee \{a\wedge b:a\in E_{x},b\in
E_{y}\}=0
\end{equation*}
because $E_{x}\cap E_{y}=\emptyset $, by point1). Finally, for all $x\in D$,
we have $\bigvee E_{x}\neq 0$, because $E_{x}\neq \emptyset $ by point 1).
\end{proof}

\begin{lemma}
\label{lemma5}If $E\leq D$ then for all $x$, $y\in D$, if $x\neq y$ then $%
x\wedge \bigvee E_{y}=0$.
\end{lemma}

\begin{proof}
We have $x\wedge \bigvee E_{y}=\bigvee \{x\wedge z:z\in E_{y}\}=0$, because $%
z\leq y$ implies $x\wedge z\leq x\wedge y=0$.
\end{proof}

\begin{theorem}
\label{teo1}If $E\leq D$ then, for all $d\in D$, $d=\bigvee E_{d}$.
\end{theorem}

\begin{proof}
We observe that 
\begin{equation*}
d\leq 1=\bigvee \{\bigvee E_{x}:x\in D\},
\end{equation*}
by point 2) of lemma \ref{lemma4}. So 
\begin{eqnarray*}
d &=&d\wedge \bigvee \{\bigvee E_{x}:x\in D\} \\
&=&\bigvee \{d\wedge \bigvee E_{x}:x\in D\} \\
&=&\bigvee E_{d}.
\end{eqnarray*}
The last line follows because $d\wedge \bigvee E_{x}=0$ when $d\neq x$, by
lemma \ref{lemma5}, and $d\wedge \bigvee E_{x}=\bigvee E_{d}$ when $x=d$, as 
$\bigvee E_{d}\leq d$ ($d$ is an upper bound for $E_{d}).$
\end{proof}

\begin{lemma}
\label{lemma6}For all $E$, $D\in \Pi (\mathcal{A})$, for all $e\in E$ there
is a $d\in D$ such that $e\wedge d\neq 0$.
\end{lemma}

\begin{proof}
We have $e\wedge \bigvee D=e\neq 0$, then $\bigvee \{e\wedge d:d\in D\}\neq
0 $ so there is a $d\in D$ such that $e\wedge d\neq 0$.
\end{proof}

\begin{theorem}
\label{teopartlatt}For all distributive finite lattice $\mathcal{A}$, $\Pi (%
\mathcal{A})$ is a bounded lattice with respect to $\leq $. If $\mathcal{A}$
is a Boolean algebra, then $At(A)$ is the bottom element of $\Pi (\mathcal{A}%
)$
\end{theorem}

\begin{proof}
Firstly we prove that $\Pi (\mathcal{A})$ is partially ordered by $\leq $.
Reflexivity and transitivity of $\leq $ are immediate. As for antisimmetry,
we suppose $E\leq D$ and $D\leq E$ and prove that $E=D$. If $e\in E$ then
there is $d\in D$ such that $e\leq d$ and $e^{\prime }\in E$ such that $%
d\leq e^{\prime }$: so $e\leq e^{\prime }$. As $e$, $e^{\prime }\in E$, if $%
e\neq e^{\prime }$ then $e\wedge e^{\prime }=0$, but $e\wedge e^{\prime }=e$
and $e\neq 0$, so we conclude that $e=e^{\prime }$. As $d$ is sandwiched
between $e$ and $e^{\prime }$, we have $d=e$, so $e\in D$. In the same way
we prove that $D\subseteq E$, so $E=D$.

Secondly we prove that $\Pi (\mathcal{A})$ is a bounded lattice with respect
to $\leq $. $\Pi (\mathcal{A})$\ has a greatest element $\{1\}$, where $1$
is the top element of $\mathcal{A}$. $\Pi (\mathcal{A})$ contains the
greatest lower bound $E\wedge D$ for all $E$, $D\in \Pi (\mathcal{A})$. For
every $E$, $D\in \Pi (\mathcal{A})$, we set 
\begin{equation*}
H=\{e\wedge d:e\in E,d\in D,e\wedge d\neq 0\}
\end{equation*}
We observe that $H$ is not empty, by the above lemma, and we prove that $%
H=E\wedge D$. In the first place we prove that $E\wedge D$ is a partition of 
$\mathcal{A}$. In fact, we have 
\begin{eqnarray*}
\bigvee \{e\wedge d:e\in E,d\in D\} &=&\bigvee \{\bigvee \{e\wedge d:d\in
D\}:e\in E\} \\
&=&\bigvee \{e\wedge \bigvee \{d:d\in D\}:e\in E\} \\
&=&\bigvee \{e:e\in E\}\wedge \bigvee \{d:d\in D\} \\
&=&1,
\end{eqnarray*}
and we have $(e\wedge d)\wedge (e^{\prime }\wedge d^{\prime })=0$ whenever $%
e $, $e^{\prime }\in E^{\prime }$and $d$, $d^{\prime }\in D$. Now we can
easily see that $H$ is $E\wedge D$. On one side, $H\leq E$, $D$ because $%
e\wedge d\leq e$ and $e\wedge d\leq d$, for all $e\wedge d\in H$. On the
other side, for all $Z\in \Pi (\mathcal{A})$ such that $Z\leq E$, $D$, we
have $Z\leq H$, because for all $z\in Z$ there are $e\in E$ and $d\in D$
such that $z\leq e$ and $z\leq d$ and then $z\leq e\wedge d$. As $\mathcal{A}
$ is finite, the bottom element of $\Pi (\mathcal{A})$ is $\bigwedge \Pi (%
\mathcal{A})$. The existence $E\vee D$ follows by theorem 2.31 of \cite
{davey2002}.)

Now we suppose that $\mathcal{A}$ is a finite Boolean algebra. Firstly, we
prove that $At(\mathcal{A})$, the set of all atoms in $\mathcal{A}$, is a
partition of $\mathcal{A}$. By definition of atom, we have $a\wedge
a^{\prime }=0$ for all $a$, $a^{\prime }\in At(\mathcal{A})$. Then we
remember that, in a finite Boolean algebra $\mathcal{A}$, \ we have $%
a=\bigvee \{x\in At(\mathcal{A}):x\leq a\}$ for all $a\in A$, (see lemma 5.4
of \cite{davey2002}) so $1=\bigvee At(\mathcal{A})$. Now we can prove that,
for all $E\in \Pi (\mathcal{A})$, $At(\mathcal{A})\leq E$: in fact, for all $%
a\in At(\mathcal{A})$ we have $a\leq \bigvee E=1$ an so there is an $e\in E$
such that $a\leq e$, by lemma 5.11 (iii) of \cite{davey2002}.
\end{proof}

In the above theorem $E\vee D$ is described as $\bigwedge \{Z\in \Pi (%
\mathcal{A}):E\leq Z$ and $D\leq Z\}$. As a result of this definition from
above, we have no idea of the inner constitution of $E\vee D$. The following
two theorems are devoted to this scope. For all $X\subseteq A$, we define $%
[X]$ as the least subalgebra of $\mathcal{A}$ including $X$, i.e. the
intersection of all $\mathcal{B}\subseteq \mathcal{A}$ such that $X\subseteq
B$. We denote with $[X]_{\vee }$ the least subset of $A$ that is closed with
respect to finite (even empty) joins.

\begin{theorem}
\label{teo2}If $\mathcal{A}$ is a finite distributive lattice and $E$
partition of $\mathcal{A}$, then

\begin{enumerate}
\item  $[E]_{\vee }=[E]$,

\item  $[E]_{\vee }$ is closed with respect to complement, i.e. $[E]_{\vee }$
is a Boolean algebra.
\end{enumerate}
\end{theorem}

\begin{proof}
1. We show that $[E]_{\vee }$ is the least subalgebra of $\mathcal{A}$
including $E$. Firstly we show that $[E]_{\vee }$ is a subalgebra of $%
\mathcal{A}$. We observe that $1\in \lbrack E]_{\vee }$ because $\bigvee E=$ 
$1$ and $0\in \lbrack E]_{\vee }$ because $\bigvee \emptyset =0$. Obviously, 
$[E]_{\vee }$ is closed with respect to $\vee $ by definition. We prove that
if $a$, $b\in $ $[E]_{\vee }$ then $a\wedge b\in \lbrack E]$. If $a=0$ or $%
b=0$ then $a\wedge b\in \lbrack E]$. Then we suppose $a\neq $, $0$ and $%
b\neq 0$. By hypothesis, there are some non empty subsets $X$, $Y\subseteq E$
such that $a=\bigvee X$ and $b=\bigvee Y$. So 
\begin{equation*}
a\wedge b=(\bigvee X)\wedge (\bigvee Y)=\bigvee \{x\wedge y:x\in X\text{, }%
y\in Y\}.
\end{equation*}
As $E$ is a partition, $x\wedge y=0$ when $x\neq y$ and $x\wedge y=x$ when $%
x=y$, so 
\begin{equation*}
a\wedge b=\left\{ 
\begin{array}{ccc}
0 & \mathrm{if} & X\cap Y=\emptyset , \\ 
\bigvee (X\cap Y) & \mathrm{if} & X\cap Y\neq \emptyset .
\end{array}
\right.
\end{equation*}
In both cases, $a\wedge b\in \lbrack E]_{\vee }$, as $X\cap Y\subseteq E$.
Trivially $E\subseteq \lbrack E]_{\vee }$. Minimality follows because, for
all $\mathcal{B}\subseteq \mathcal{A}$ such that $E\subseteq B$, $[E]_{\vee
}\subseteq \mathcal{B}$.

2. We prove that for all $a\in \lbrack E]_{\vee }$ there is an element $%
\lnot a$ that is the complement of $a$ in $[E]_{\vee }$. We know that, by
hypothesis, $a=\bigvee X$ for some $X\subseteq E$, so we set $\lnot
a=\bigvee (E-X)$. Then we have 
\begin{equation*}
a\vee \lnot a=\bigvee X\vee \bigvee (E-X)=\bigvee E=1
\end{equation*}
and 
\begin{eqnarray*}
a\wedge \lnot a &=&\bigvee X\wedge \bigvee (E-X) \\
&=&\bigvee \{x\wedge y:x\in X\text{, }y\in E-X\} \\
&=&0,
\end{eqnarray*}
as $X\cap (E-X)=\emptyset $ and $x\wedge y=0$ when $x\neq y$.
\end{proof}

\begin{theorem}
\label{teo3}If $E$, $D\in \Pi (\mathcal{A})$ then $E\leq D$ iff $%
[D]\subseteq \lbrack E]$.
\end{theorem}

\begin{proof}
We assume $E\leq D$. If $x\in \lbrack D]$ then $x=\bigvee D^{\prime }$ for
some $D^{\prime }\subseteq D$. For all $d^{\prime }\in D^{\prime }$, we have 
$d^{\prime }=\bigvee E_{d^{\prime }}$ by \ref{teo1}, so we have $x=\bigvee
\{\bigvee E_{d^{\prime }}:d^{\prime }\in D^{\prime }\}$ and $x\in \lbrack E]$%
. We assume $[D]\subseteq \lbrack E]$. Then $D\subseteq \lbrack E]$ and $%
\bigvee D\in \lbrack E]$, so for all $e\in E$ we have $e\leq 1=\bigvee D$.
As every $e\in E$ is $\vee $-irreducible in $[E]$, there is a $d\in D$ such
that $e\leq d$, by lemma 5.11 of \cite{davey2002}, so $E\leq D$.
\end{proof}

Now we can give a more constructive description of $E\vee D$.

\begin{theorem}
\label{teopartsup}If $\mathcal{A}$ is a finite distributive lattice and $E$, 
$D$ are partitions of $\mathcal{A}$, then $E\vee D=At([E]\cap \lbrack D])$.
\end{theorem}

\begin{proof}
Firstly we observe that $[E]$, $[D]$ are subalgebras of $\mathcal{A}$, so $%
[E]\cap \lbrack D]$ is also a subalgebra of $\mathcal{A}$. As $[E]$, $[D]$
are Boolean algebras, by \ref{teo2}, so is $[E]\cap \lbrack D]$ and we can
speak of $At([E]\cap \lbrack D])$. As $At([E]\cap \lbrack D])$ is a
partition of $[E]\cap \lbrack D]$ and $[E]\cap \lbrack D]$ is a subalgebra
of $\mathcal{A}$, $At([E]\cap \lbrack D])$ is a partition of $\mathcal{A}$.
Firstly, we show that $E\leq At([E]\cap \lbrack D])$. In fact, $[At([E]\cap
\lbrack D])]=[E]\cap \lbrack D]\subseteq \lbrack E]$ so, by theorem \ref
{teo3}, we can conclude that $E\leq At([E]\cap \lbrack D])$. In the same way
we prove that $D\leq At([E]\cap \lbrack D])$. Finally, we prove that, for
all $G$ such that $E\leq G$ and $D\leq G$, we have $At([E]\cap \lbrack
D])\leq G$. From our hypothesis, by theorem \ref{teo3}, $\ [G]\subseteq
\lbrack E]$ and [$G]\subseteq \lbrack D]$, so $[G]\subseteq \lbrack E]\cap
\lbrack D]=[At([E]\cap \lbrack D])]$ and by theorem \ref{teo3} we have $%
At([E]\cap \lbrack D])\leq $ $G$.
\end{proof}

The following theorem shows that every element of $[E]$, different from $0$,
is uniquely generated by $\bigvee $ from $E$.

\begin{theorem}
\label{teoungen}For all $E\in \Pi (A)$, if $X$, $Y\subseteq E$ and $\bigvee
X=\bigvee Y$, then $X=Y$.
\end{theorem}

\begin{proof}
If $x\in X$ then $x\leq \bigvee X=\bigvee Y$. If for all $y\in Y$ we have $%
x\wedge y=0$, then $x=x\wedge \bigvee Y=\bigvee \{x\wedge y:y\in Y\}=0$, but
this is absurd because $x\in E$ and $E$ is a partition, so there is $y\in Y$
such that $x\wedge y\neq 0$. As $x$, $y\in E$, this implies $x=y$, so $x\in
Y $ and $X\subseteq Y$. In the same way we prove that $Y\subseteq X$.
\end{proof}

\section{Allais Paradox and intrinsic expected value\label{Appendix B}}

A decision problem leading to a somewhat paradoxical conclusion has been
presented by Maurice Allais in 1953 (see \cite{allais1953}). The acts
involved are $\alpha $, $\alpha ^{\prime }$, $\beta $ and $\beta ^{\prime }$%
, all having a common domain $E=\{e_{1},e_{2},e_{3}\}$. The probabilities of
the three conditions are $p(e_{1})=0.01$, $p(e_{2})=0,1$ and $p(e_{3})=0.89$%
. The rewards, in dollars, are:

\begin{center}
\begin{tabular}{l|l|l|l}
& $e_{1}$ & $e_{2}$ & $e_{3}$ \\ \cline{1-4}
$\alpha $ & $500000$ & $500000$ & $500000$ \\ \cline{1-4}
$\alpha ^{\prime }$ & $0$ & $2500000$ & $500000$%
\end{tabular}
\hspace{10 mm} 
\begin{tabular}{l|l|l|l}
& $e_{1}$ & $e_{2}$ & $e_{3}$ \\ \cline{1-4}
$\beta $ & $500000$ & $500000$ & $0$ \\ \cline{1-4}
$\beta ^{\prime }$ & $0$ & $2500000$ & $0$%
\end{tabular}
\end{center}

\noindent The decision maker must choose between $\alpha $ and $\alpha
^{\prime }$ and between $\beta $ and $\beta ^{\prime }$: if he maximizes
expected utility, then $\alpha ^{\prime }$ is better than $\alpha $ and $%
\beta ^{\prime }$ is better then $\beta $, but empirical evidence shows that
a great many people prefer $\alpha $ to $\alpha ^{\prime }$ and $\beta
^{\prime }$ to $\beta $. So maximizing expected value cannot be considered
as an universal rule of choice between acts. This situation is generally
explained by observing that the choice between $\alpha $ and $\alpha
^{\prime }$ is a decision problem qualitatively different from the choice
between $\beta $ and $\beta ^{\prime }$ The choice of $\alpha $ stems from
risk aversion, because $\alpha $ is a constant function that banishes every
aleatory aspect, so the decision maker leaves out any question about
probability and expected value. On the other side, $\beta $ and $\beta
^{\prime }$ are both risky acts and consideration of expected value is
appropriate. Before going farther in the analysis of Allais Paradox, we
introduce some concepts of general character.

When $X$ and $Y$ are partially ordered sets, we say that $f:X\rightarrow Y$
is an \textit{order embedding} when $x\leq x^{\prime }$ in $X$ iff $f(x)\leq
f(x^{\prime })$ in $Y$. (It can be easily seen that every order embedding is
injective.) If $f:A(E)\rightarrow A(E)$ is an order embedding, then any
decision problem about acts in $A(E)$ can be reduced to a decision problem
about acts in $f[A(E)]$, in the following sense. If $\leq $ represents
desirability of acts and we are asked if $\alpha \leq \beta $, then we can
shift the problem to desirability of $f(\alpha )$ and $f(\beta )$: if we
find that $f(\alpha )\leq f(\beta )$, then also $\alpha \leq \beta $. We
focus on the family of order embeddings associated to positive affine
transformations of $R$. For all pair of real numbers $h>0$ and $k$, $\tau
(x)=hx+k$ is the positive affine transformation associated to $(h,k)$. Now
we consider the function $f:A(E)\rightarrow A(E)$ such that $f(\alpha )=\tau
\circ \alpha $: we have, for all $e\in E$, $f(\alpha )(e)=\tau (\alpha
(e))=h\alpha (e)+k$.

\begin{theorem}
If $f:A(E)\rightarrow A(E)$, where $f(\alpha )=\tau \circ \alpha $ and $\tau 
$ is the positive affine transformations $\tau (x)=hx+k$, then

\begin{enumerate}
\item  $f$ is an order embedding of $A(E)$ in itself,

\item  for all $\alpha \in A(E)$ and all valuation $v:\mathcal{A}\rightarrow
\lbrack 0,1]$, $\exp (\tau \circ \alpha ,v)=\tau (\exp (\alpha ,v))$,

\item  $\exp (\alpha ,v)\leq \exp (\alpha ^{\prime },v)$ iff $\exp (f(\alpha
),v)\leq \exp (f(\alpha ^{\prime }))$.
\end{enumerate}
\end{theorem}

\begin{proof}
1. We have $\alpha \leq \beta $ iff, for all $e\in E$, $\alpha (e)\leq \beta
(e)$ iff $h\alpha (e)+k\leq h\beta (e)+k$ iff $f(\alpha )\leq f(\beta )$.

2. We have 
\begin{eqnarray*}
\exp (\tau \circ \alpha ,v) &=&\sum \{(h\alpha (e)+k)v(e):e\in E\} \\
&=&\sum \{h\alpha (e)v(e)+kv(e):e\in E\} \\
&=&h\sum \{\alpha (e)v(e):e\in E\}+k\sum \{v(e):e\in E\} \\
&=&h\exp (\alpha ,v)+k.
\end{eqnarray*}

3.Trivial.
\end{proof}

The same kind of order embedding can be defined from $A(\mathcal{A})$ to $A(%
\mathcal{A})$ and a similar theorem can be proved.

Now we can give an equivalent formulation of Allais Paradox by defining four
acts as follows:

\begin{center}
\begin{tabular}{l|l|l|l}
& $e_{1}$ & $e_{2}$ & $e_{3}$ \\ \cline{1-4}
$f(\alpha )$ & $1$ & $1$ & $1$ \\ \cline{1-4}
$f(\alpha ^{\prime })$ & $0$ & $5$ & $1$%
\end{tabular}
\hspace{10 mm} 
\begin{tabular}{l|l|l|l}
& $e_{1}$ & $e_{2}$ & $e_{3}$ \\ \cline{1-4}
$f(\beta )$ & $1$ & $1$ & $0$ \\ \cline{1-4}
$f(\beta ^{\prime })$ & $0$ & $5$ & $0$%
\end{tabular}
\end{center}

The transformation involved is $\tau (x)=1/500000x$. As a consequence of
point 1) in the above theorem, we have $\exp (f(\alpha ),v)\leq $ $\exp
(f(\alpha ^{\prime }),v)$ and $\exp (f(\beta ),v)\leq $ $\exp (f(\beta
^{\prime }),v)$, but we may still prefer $f(\alpha )$ to $f(\alpha ^{\prime
})$ by risk aversion, as in the original formulation of Allais Paradox. We
underline that only an equivalence of mathematical character is discussed
here: $f$ preserves $\leq $, but may not preserve the psychological impact
of acts.

We can give a more abstract formulation of Allais Paradox with the following
acts:

\begin{center}
\begin{tabular}{l|l|l|l}
& $e_{1}$ & $e_{2}$ & $e_{3}$ \\ \hline
$\alpha $ & $x$ & $x$ & $x$ \\ \hline
$\alpha ^{\prime }$ & $0$ & $y$ & $x$%
\end{tabular}
\hspace{10mm} 
\begin{tabular}{l|l|l|l}
& $e_{1}$ & $e_{2}$ & $e_{3}$ \\ \hline
$\beta $ & $x$ & $x$ & $0$ \\ \hline
$\beta ^{\prime }$ & $0$ & $y$ & $0$%
\end{tabular}
\end{center}

\noindent We assume $x$, $y>0$. The particular case above is obtained
setting $x=1$ and $y=5$, but not every choice of $x$ and $y$ gives place to
an instance of Allais Paradox. By definition of the four acts, we have $\exp
(\alpha ,v)\leq \exp (\alpha ^{\prime },v)$ iff $\exp (\beta ,v)\leq \exp
(\beta ^{\prime },v)$, but in particular we should also have $\exp (\alpha
,v)<\exp (\alpha ^{\prime },v)$. We observe that 
\begin{eqnarray*}
\exp (\alpha ,v)<\exp (\alpha ^{\prime },v) &&\text{ iff }x\cdot 0.01+x\cdot
0.1+x\cdot 0.89<y\cdot 0.1+x\cdot 0.89 \\
&&\text{ }\text{ iff }x\cdot 1.1<y\text{,}
\end{eqnarray*}
so we can conclude that $\alpha ^{\prime }$ is better then $\alpha $ as far
as $y>x+\frac{1}{10}x$. Thus Allais Paradox can be sharpened by choosing a
value of $y$ just a little bigger than $x+\frac{1}{10}x$: if risk aversion
works when $y$ is five times $x$, it should also be at work with a lesser
value of $y$. As for $\beta $ and $\beta ^{\prime }$, there is no risk
aversion because both are risky acts. But there is still another aspect of
acts that can be considered in the analysis of Allais Paradox.

The consequences of acts are real numbers and we can rather naively say that
big numbers correspond to big rewards, but the meaning of `big' is dependent
from the context. We assume that the context of an act is the total sum of
rewards, so a reward is a really big one if it is a relevant part of this
total. For all act $\alpha :E\rightarrow R$, we define $T(\alpha )=\sum
\{\alpha (e):e\in E\}$ and we call $T(\alpha )$ the \textit{total} of $%
\alpha $. When $T(\alpha )>0$, we define an act $\bar{\alpha}:E\rightarrow R$
setting $\bar{\alpha}(e)=\frac{1}{T(\alpha )}\alpha (e)$: we call $\bar{%
\alpha}$ the \textit{standardization} of $\alpha $. Clearly, $T(\bar{\alpha}%
)=1$, because $\frac{1}{T(\alpha )}\sum \{\bar{\alpha}(e):e\in E\}=1$. If $%
\alpha (e)\geq 0$, for all $e\in E$, then $\bar{\alpha}[E]\subseteq \lbrack
0,1]$. Finally, we define 
\begin{equation*}
\overline{\exp }(\alpha ,v)=\exp (\bar{\alpha},v)
\end{equation*}
and call $\overline{\exp }(\alpha ,v)$ the \textit{intrinsic expected value}
of $\alpha $. In the original form of Allais Paradox, we have $\overline{%
\exp }(\alpha ,v)>\overline{\exp }(\alpha ^{\prime },v)$ and $\overline{\exp 
}(\beta ,v)<\overline{\exp }(\beta ^{\prime },v)$: if the decision maker
maximizes the intrinsic expected value, then $\alpha $ is preferred to $%
\alpha ^{\prime }$ and $\beta ^{\prime }$ to $\beta $. So the paradox can be
explained non only by risk aversion, but also by a different valuation of
the performance of an act, based on the ratio between the expected value and
the total amount of possible rewards. As a consequence of point 3) of the
following theorem, this kind of solution works for all the instances of the
original form of the Allais Paradox obtained by a similarity, i.e. a
positive affine transformation $\tau (x)=hx+k$ with $k=0$.

\begin{theorem}
If $f:A(E)\rightarrow A(E)$, where $f(\alpha )=\tau \circ \alpha $ and $\tau
(x)=hx$, with $k>0$, then:

\begin{enumerate}
\item  $\exp (\bar{\alpha},v)=\frac{1}{T(\alpha )}\exp (\alpha ,v)$,

\item  $\overline{\exp }(\tau \circ \alpha ,v)=\overline{\exp }(\alpha ,v)$,

\item  $\overline{\exp }(\alpha ,v)\leq \overline{\exp }(\alpha ^{\prime },v)
$ iff $\overline{\exp }(f(\alpha ),v)\leq \overline{\exp }(f(\alpha ^{\prime
}))$.
\end{enumerate}
\end{theorem}

\begin{proof}
1. We have
\begin{equation*}
\sum \{\bar{\alpha}(e)v(e):e\in E\}=\frac{1}{T(\alpha )}\sum \{\bar{\alpha}%
(e)v(e):e\in E\}=\frac{1}{T(\alpha )}\exp (\alpha ,v).
\end{equation*}

2. We have $T(\tau \circ \alpha )=\sum \{h\alpha (e):e\in E\}=h\sum \{\alpha
(e):e\in E\}=hT(\alpha )$. So 
\begin{eqnarray*}
\overline{\exp }(\tau \circ \alpha ,v) &=&\exp (\overline{\tau \circ \alpha }%
,v)=\frac{1}{T(\tau \circ \alpha )}\exp (\tau \circ \alpha ,v)=\frac{1}{%
hT(\alpha )}h\exp (\alpha ,v) \\
&=&\frac{1}{T(\alpha )}\exp (\alpha ,v)=\exp (\bar{\alpha},v).
\end{eqnarray*}

3. Trivial.
\end{proof}

If we take in account the abstract form of Allais Paradox, we have 
\begin{eqnarray*}
\overline{\exp }(\alpha ^{\prime },v)<\overline{\exp }(\alpha ^{\prime },v) &%
\text{ iff }&\frac{1}{T(\alpha ^{\prime })}\exp (\alpha ^{\prime },v)<\frac{1%
}{T(\alpha )}\exp (\alpha ,v) \\
&\text{ iff }&\frac{1}{x+y}(y\cdot 0.1+x\cdot 0.89)<\frac{1}{3} \\
&\text{ iff }&2.3857x<y\text{,}
\end{eqnarray*}
while on the other side we have always $\overline{\exp }(\beta ,v)<\overline{%
\exp }(\beta ^{\prime },v)$.


\begin{thebibliography}{9}
\bibitem{allais1953}  M. Allais, Le comportement de l'homme rationnel devant
lo risque, \textit{Econometrica}, 21, 503-546, 1953.

\bibitem{birkhoff1967}  G. Birkhoff, \textit{Lattice Theory}, A.M.S. v. 25,
1967.

\bibitem{ck1990}  C.C. Chang, H.J. Keisler, \textit{Model Theory},
North-Holland, 1990.

\bibitem{davey2002}  B. A. Davey, H. A. Priestley, \textit{Introduction to
Lattices and Order}, Cambridge U.P.,2002 (second edition).

\bibitem{ellsberg1961}  D. Ellsberg, Risk, ambiguity and the Savage axioms, 
\textit{Quarterly Journal of Economics}, 75, 643-669, 1961.

\bibitem{jeffrey1983}  R. C. Jeffrey, \textit{The logic of decision},
University of Chicago Press, 1983.

\bibitem{negri2013}  M. Negri, Partial probability and Kleene Logic, \textit{%
arXiv}:1310.6172, 2013.
\end{thebibliography}
\end{document}